\documentclass[twoside]{article}

\usepackage[accepted]{aistats2020}
\usepackage[utf8]{inputenc} 
\usepackage[T1]{fontenc}    
\usepackage[colorlinks=true,linkcolor=blue,citecolor=cyan,]{hyperref}       
\usepackage{url}            
\usepackage{booktabs}       
\usepackage{amsfonts}       
\usepackage{nicefrac}       
\usepackage{microtype}      

\usepackage{graphicx}
\usepackage{amsfonts}
\usepackage{amsthm}
\usepackage{amsmath}
\usepackage{amssymb}
\usepackage{amscd}
\usepackage{parskip}
\usepackage{enumitem}
\usepackage{verbatim}
\usepackage{bibentry}
\usepackage{natbib}
\usepackage{booktabs}
\usepackage{thm-restate}

\usepackage{subcaption}
\usepackage{xcolor}
\usepackage[framemethod=TikZ]{mdframed}

\newtheoremstyle{thmstyle}
  {5pt} 
  {\topsep} 
  {} 
  {} 
  {\bfseries} 
  {.} 
  {.5em} 
  {} 
\theoremstyle{thmstyle}

\newtheorem{theorem}{Theorem}[section]

\newtheorem{definition}[theorem]{Definition}
\newtheorem{proposition}[theorem]{Proposition}

\usepackage{wrapfig}

\usepackage{algorithm}
\usepackage{algorithmic}

\newcommand{\beq}{\begin{equation}}
\newcommand{\eeq}{\end{equation}}

\newcommand{\beqa}{\begin{eqnarray}}
\newcommand{\eeqa}{\end{eqnarray}}

\newcommand{\beqan}{\begin{eqnarray*}}
\newcommand{\eeqan}{\end{eqnarray*}}

\usepackage{selectp}

\newcommand{\statespace}{\mathcal{X}}
\newcommand{\actionspace}{\mathcal{A}}

\usepackage{xspace}

\newcommand{\scflong}{sufficient conditioning functional\xspace}
\newcommand{\scfslong}{sufficient conditioning functionals\xspace}

\newcommand{\SCF}{SCF\xspace}
\newcommand{\SCFs}{SCFs\xspace}
\newcommand{\scsalong}{sufficient conditioning sigma-algebra\xspace}
\newcommand{\SCSA}{SCSA\xspace}
\newcommand{\CIW}{CIS\xspace}
\newcommand{\ciw}{conditional importance sampling\xspace}

\newcommand{\sciw}{state-conditioned importance sampling\xspace}

\newcommand{\SCIW}{SCIS\xspace}

\newcommand{\trajectory}[1]{\eta^{#1}}
\newcommand{\partialtrajectory}[2]{\eta^{#1}_{#2}}

\newcommand{\ISreturnE}[3]{\bar{G}^{\text{OIS};#1,#2}_{#3}}
\newcommand{\PDISreturnE}[3]{\bar{G}^{\text{PDIS};#1,#2}_{#3}}

\usepackage{selectp}

\usepackage{tikz-cd}

\usepackage{mathtools}
\newcommand{\defeq}{\vcentcolon=}

\usepackage{mathrsfs}

\usepackage{verbatim}

\begin{document}

%

%
\runningauthor{Rowland, Harutyunyan, van Hasselt, Borsa, Schaul, Munos, Dabney}

\twocolumn[

\aistatstitle{Conditional Importance Sampling for Off-Policy Learning}

\aistatsauthor{Mark Rowland \ \ \ \ \ \ \ \ \  Anna Harutyunyan \ \ \ \ \ \ \ \ \ Hado van Hasselt \\ \textbf{Diana Borsa \ \ \ \ \ \ \ \ \ Tom Schaul \ \ \ \ \ \ \ \ \  R{\'e}mi Munos \ \ \ \ \ \ \ \ \  Will Dabney}}

\aistatsaddress{DeepMind} ]

\begin{abstract}
  The principal contribution of this paper is a conceptual framework for off-policy reinforcement learning, based on conditional expectations of importance sampling ratios. This framework yields new perspectives and understanding of existing off-policy algorithms, and reveals a broad space of unexplored algorithms. We theoretically analyse this space, and concretely investigate several algorithms that arise from this framework.
\end{abstract}

\section{Introduction}

Using off-policy data is crucial for many tasks in reinforcement learning (RL), including for acquiring knowledge about diverse aspects of the environment \citep{sutton2011horde}, learning from memorised data \citep{mnih2015humanlevel,schaul2016prior}, 
exploration \citep{watkins1992q}, 
and learning to perform auxiliary tasks \citep{schaul2015universal,jaderberg2017reinforcement,bellemare2019geometric}.
One of the fundamental techniques for correcting for the difference between the policy that generated the data and the policy that an algorithm aims to learn about is importance sampling (IS) \citep{MET49,Kahn1949}, which was first introduced in off-policy RL by \citet{precup2000}.
Importance sampling features as a core ingredient of many off-policy algorithms~\citep{maei2011gradient,van2014off,munos2016safe,jiang2016doubly,sutton2016emphatic}, and is supported by strong theoretical understanding coming from the computational statistics literature \citep{robert2013monte,sarkka2013bayesian}.

Importance sampling often suffers from high variance, especially when multi-step trajectories are considered. This has motivated the study of a wide range of variance reduction techniques in off-policy reinforcement learning. These techniques include importance weight truncation \citep{munos2016safe,pmlr-v80-espeholt18a} weighted importance sampling \citep{precup2000,Mahmood:2014}, adaptive bootstrapping \citep{mahmood2017multi}, variants of emphatic TD \citep{hallak2016generalized}, saddle-point formulations exploiting low-variance versions of SGD \citep{du2017stochastic,johnson2013accelerating,defazio2014saga}, empirical proposal estimation \citep{hanna2019importance}, doubly-robust approaches \citep{jiang2016doubly,thomas2016data}, confidence bounds on returns \citep{thomas2015high,thomas2015highimprovement,metelli2018policy,papini2019optimistic} and state distribution estimation \citep{MIS,liu2018breaking,kallus2019double,kallus2019efficiently,uehara2019minimax,hallak2017consistent,gelada2019off,nachum2019dualdice}.

In this paper, we propose a new framework for variance reduction in off-policy learning, \emph{conditional importance sampling} (\CIW), based on taking conditional expectations of importance weights. This framework is motivated by the observation that when estimating a return off-policy using standard importance sampling, every action along a trajectory contributes to the importance weight, even if the action had no effect on the return observed. Intuitively, it would be preferable for the importance weight to depend only on the return itself; if two policies generate similar distributions of returns, there should be no need to perform importance weighting at all. As just one application of the \CIW framework, we make this insight precise, and introduce \emph{return-conditioned importance sampling} (RCIS), a new off-policy evaluation algorithm. Concretely, using notation introduced formally in Section~\ref{sec:is}, given a random return $G$, RCIS uses \emph{conditional} importance weights of the form
\begin{align*}
    \mathbb{E}\left\lbrack \prod_{t=1}^{n-1} \frac{\pi(A_t|X_t)}{\mu(A_t|X_t)}\ \middle|\ G \right\rbrack \, ,
\end{align*}
which integrates out noise in the trajectory that is irrelevant in determining the return, leading to a lower-variance importance weight.

However, return is just one possible variable to condition on. The central insight of the \CIW framework is that there exists a large space of variables that the importance weights can be conditioned on, with each choice leading to a different off-policy algorithm. In the remainder of the paper, we give a mathematical description of the general \CIW framework, which then allows us to make several further contributions:
\begin{enumerate}[leftmargin=0.6cm,label={(\roman*)},topsep=-4pt,itemsep=0pt,partopsep=0pt,parsep=0pt]
    \item We compare and analyse the statistical properties of \CIW algorithms based on properties of the conditioning variables.
    \item We study several specific instantiations of algorithms from this framework, including RCIS and \emph{\sciw} (\SCIW, given by conditioning on the states visited by a trajectory at each timestep).
    \item We develop practical versions of these algorithms, based on learning the conditional importance weights in a supervised manner.
\end{enumerate}

We note that concurrently with this work, \cite{liu2019understanding} also consider conditional importance sampling in off-policy learning, establishing connections with the conditional Monte Carlo literature and undertaking statistical analysis of these estimators.

\section{Background}\label{sec:is}

Consider a Markov decision process (MDP) $(\statespace,\actionspace,\gamma,P,\mathcal{R})$ with finite state space $\statespace$, finite action space $\actionspace$, discount factor $\gamma \in [0,1)$, transition kernel $P:\statespace\times\actionspace \rightarrow \mathscr{P}(\statespace)$, reward distribution probability mass function $\mathcal{R}: \mathbb{R}\times\statespace\times\actionspace\rightarrow\mathbb{R}$ (so that $\mathcal{R}(r, x, a)$ encodes the probability of observing reward $r$ after taking action $a$ in state $x$), and initial state distribution $\nu \in \mathscr{P}(\statespace)$\footnote{With some care, it is possible to show through the use of measure theory that versions of many results in this paper hold in much greater generality, such as in classes of MDPs with continuous state and/or action spaces. For the sake of accessibility and clarity of exposition, the main paper focuses on the discrete case, but we discuss how these results generalise in Appendix~\ref{sec:generalising} for the interested reader.}.

Given a Markov policy $\pi:\statespace\rightarrow\mathscr{P}(\actionspace)$, the distribution of the process $(X_t, A_t, R_t)_{t \geq 0}$ itself is defined by $X_0 \sim \nu$, $A_t | X_{0:t},A_{0:t-1},R_{0:t-1} \sim \pi(\cdot|X_t)$, $R_t | X_{0:t},A_{0:t},R_{0:t-1} \sim \mathcal{R}(\cdot|X_t,A_t)$, and $X_{t+1}|X_{0:t},A_{0:t},R_{0:t} \sim P(\cdot|X_{t},A_{t})$ for each $t \geq 0$.
We denote the full trajectory $(X_t,A_t,R_t)_{t \geq 0}$ by $\tau$, and use the notation $\tau_{s:t}$ to denote the partial trajectory $(X_s,A_s,R_s,X_{s+1},\ldots,X_t)$. 
We denote the distribution of $\tau$ under the policy $\pi$ by $\trajectory{\pi}$, 
and denote the distribution of $\tau_{s:t}$ by $\partialtrajectory{\pi}{s:t}$ for any $0\leq s \leq t$. We will also denote conditional versions of these distributions given $(X_0,A_0)=(x,a)$ in the manner $\trajectory{\pi}|_{(x,a)}$.

\subsection{Policy evaluation}\label{sec:pe}

The \emph{evaluation problem} with \emph{target policy} $\pi : \statespace \rightarrow \mathscr{P}(\actionspace)$
is defined as estimation of the Q-function
\begin{align}\label{eq:Q}
    Q^\pi(x, a) \defeq \mathbb{E}_{\trajectory{\pi}|_{(x, a)}}\!\left\lbrack \sum_{t=0}^\infty \gamma^t R_t \right\rbrack \, ,
\end{align}
for all $(x, a) \in \statespace \times \actionspace$.
The fundamental result of value-based RL is that the Q-function in Expression~\eqref{eq:Q} satisfies the \emph{Bellman equation} $T^\pi Q^\pi = Q^\pi$ \citep{bellman1957}, where the one-step Bellman evaluation operator $T^\pi : \mathbb{R}^{\statespace\times\actionspace} \rightarrow \mathbb{R}^{\statespace\times\actionspace}$ is defined by
\begin{align*}
    (T^\pi Q)(x, a)\! = \! \mathbb{E}_{\trajectory{\pi}|_{(x, a)}}\!\left\lbrack R_0 \!+\! \gamma  Q(X_1,A_1) \right\rbrack ,
\end{align*}
for all $Q \in \mathbb{R}^{\statespace\times\actionspace}$ and $(x,a) \in \statespace\times\actionspace$. As $T^\pi$ is a contraction in $(\mathbb{R}^{\statespace\times\actionspace}, \|\cdot\|_{\infty})$, $Q^\pi$ is its \emph{unique} fixed point, and repeated application of $T^\pi$ to any initial Q-function will converge to $Q^\pi$. An evaluation algorithm may therefore seek to (approximately) perform a recursion of the form $Q_{k+1} \leftarrow T^\pi Q_k$ ($k \geq 1$), with the aim of converging to $Q^\pi$. More general classes of contractive operators with fixed point $Q^\pi$ can also be considered, such as the Retrace operator \citep{munos2016safe}, and the $n$-step Bellman operator, given by
\begin{align*}
    ((T^\pi)^n Q)(x, a) \! =\! \mathbb{E}_{\trajectory{\pi}|_{(x, a)}}\!\left\lbrack \sum_{t=0}^{n-1} \gamma^t R_t \!+\! \gamma^n  Q(X_n,A_n) \right\rbrack .
\end{align*}

\subsection{Off-policy policy evaluation}

Exact computation of the expectations defining the above operators is often intractable, and so Monte Carlo\footnote{Throughout, we use the term ``Monte Carlo'' in its statistical sense, to mean sampled-based approximation of any expectation, including those defining temporal difference algorithms.} estimators based on trajectories sampled from the environment are used \citep{bertsekas1996neuro,szepesvari2010algorithms,Sutton2018}. Further, it is often desirable, or necessary, to use trajectories sampled from a different distribution $\trajectory{\mu}$, based on a \emph{behaviour policy} $\mu : \statespace \rightarrow \mathscr{P}(\mathcal{A})$; in such cases, the problem is said to be \emph{off-policy}.

A common estimator for the application of the $n$-step Bellman operator $(T^\pi)^n$ to a Q-function $Q$ at a specific state-action pair $(x, a) \in \statespace\times\actionspace$ is given by sampling $\tau_{0:n}$ from $\partialtrajectory{\mu}{0:n}|_{(x, a)}$, and computing a \emph{bootstrapped return}, defined by
\begin{align}\label{eq:Ebootstrapreturn}
    \bar{G}^\pi_{0:n} \defeq \sum_{t=0}^{n-1} \gamma^{t} R_t + \gamma^{n} V(X_n; \pi) \, ,
\end{align}
where $V(x; \pi) = \mathbb{E}_{A \sim \pi(\cdot|x)}[Q(x, A)]$, 
and an \emph{importance}-\emph{weighting correction term}, defined by
\begin{align}\label{eq:traj-iw}
    \rho^{\pi,\mu}_{s:t} \defeq \prod_{i=s}^t \frac{\pi(A_i|X_i)}{\mu(A_i|X_i)} \, .
\end{align}
for $1\leq s \leq t$, 
and finally forming the \emph{ordinary importance sampling} (OIS) estimator
\begin{align}\label{eq:ois}
     \ISreturnE{\pi}{\mu}{0:n} \defeq \rho_{1:n-1}^{\pi,\mu} \bar{G}^\pi_{0:n} \, .
\end{align}

Much research in off-policy learning is concerned with constructing such estimators that have desirable statistical properties, such as low variance and consistency. Throughout, we will assume the \emph{support condition}:
\begin{align}\label{eq:sc}\tag{SC}
    \text{supp}(\pi(\cdot|x)) \subseteq \text{supp}(\mu(\cdot|x)) \ \text{ for all } x \in \statespace \, ,
\end{align}
a mild assumption that is sufficient for unbiased importance sampling, which is satisfied by exploratory behaviours such as $\varepsilon$-greedy. 
This is equivalent to absolute continuity of $\pi$ with respect to $\mu$ at each state; intuitively, this ensures that any trajectory that can arise by following $\pi$ is also realisable under $\mu$.

\section{Preliminary analysis}\label{sec:isinrl}

As a warm-up and motivation for the conceptual framework we present in the next section, we analyse some commonly-used off-policy Monte Carlo estimators.

\subsection{Ordinary importance sampling}

We begin with a formal proof of the unbiasedness of the OIS estimator, a well-known result in the literature. In this and many results that follow, we will be interested in distributions over trajectories conditioned on some initial state-action pair $(x, a) \in \statespace \times \actionspace$; this will be present in the notation, but we avoid continuously mentioning it in the text for brevity. We examine the proof of this result in some detail, since it will be informative for the original results that follow. Proofs of other results in the paper are given in Appendix~\ref{sec:proofs}.

\begin{proposition}\label{prop:nstepunbiased}
    Assume the support condition \eqref{eq:sc} holds. 
    For a trajectory drawn from $\trajectory{\mu}|_{(x, a)}$, the OIS estimator in Expression \eqref{eq:ois} is unbiased for the output of the $n$-step return operator $(T^\pi)^n$. 
    That is,
    \begin{align}\label{eq:isestimators}
        \mathbb{E}_{\trajectory{\mu}|_{(x, a)}}\!\left\lbrack \rho^{\pi,\mu}_{1:n-1} \bar{G}^\pi_{0:n} \right\rbrack 
        = 
        \mathbb{E}_{\trajectory{\pi}|_{(x, a)}}\!\left\lbrack \bar{G}^\pi_{0:n} \right\rbrack \, .
    \end{align}
\end{proposition}

\begin{proof}
We first observe that the ratio of policy probabilities that appears within the factor $\rho^{\pi,\mu}_{1:n-1}$ can also be interpreted as the importance ratio for the conditional trajectory distributions $\partialtrajectory{\mu}{0:n}|_{(x,a)}$ and $\partialtrajectory{\pi}{0:n}|_{(x,a)}$, as the following calculation shows:
\begin{align}
    &\frac{\partialtrajectory{\pi}{0:n}|_{(x,a)}(\tau_{0:n})}{\partialtrajectory{\mu}{0:n}|_{(x,a)}(\tau_{0:n})} \label{eq:trajprobs}\\
     = &
    \frac{P(X_1|x_0,a_0) \mathcal{R}(R_0|x_0,a_0)}{P(X_1|x_0,a_0)\mathcal{R}(R_0|x_0,a_0)} \times \nonumber \\
    & \ \ \frac{\prod_{t=1}^{n-1} \pi(A_t|X_t) \mathcal{R}(R_t|X_t,A_t) P(X_{t+1}|X_t, A_t)}{\prod_{t=1}^{n-1} \mu(A_t|X_t) \mathcal{R}(R_t|X_t,A_t) P(X_{t+1}|X_t, A_t)} \nonumber \\
     = & 
    \prod_{t=1}^{n-1} \frac{\pi(A_t|X_t)}{\mu(A_t|X_t)} \label{eq:actionprobs} \\
     = & \rho^{\pi,\mu}_{1:n-1} \nonumber
    \, .
\end{align}
 
Noting also that the term $\bar{G}^\pi_{0:n}$ in Equation~\eqref{eq:isestimators} is simply a function of the random truncated trajectory $\tau_{0:n}$, we may now appeal to standard importance sampling theory, using the notation $\Psi(\tau_{0:n}) = \bar{G}^\pi_{0:n}$, to obtain
\begin{align*}
    \mathbb{E}_{\trajectory{\mu}|_{(x, a)}}\!\!\left\lbrack \rho^{\pi,\mu}_{1:n-1} \bar{G}^\pi_{0:n} \right\rbrack
    \!= & \mathbb{E}_{\trajectory{\mu}|_{(x, a)}}\!\!\left\lbrack \frac{\partialtrajectory{\pi}{0:n}|_{(x, a)}(\tau_{0:n})}{\partialtrajectory{\mu}{0:n}|_{(x, a)}(\tau_{0:n})} \Psi(\tau_{0:n}) \right\rbrack \\
    = & \mathbb{E}_{\trajectory{\pi}|_{(x, a)}}\!\left\lbrack \Psi(\tau_{0:n}) \right\rbrack \\
    = & \mathbb{E}_{\trajectory{\pi}|_{(x, a)}}\!\left\lbrack \bar{G}^\pi_{0:n} \right\rbrack
    \, ,
\end{align*}
as required.
\end{proof}

We highlight two points. Firstly, note that the argument above did not depend on any special structure of $\bar{G}^\pi_{0:n}$, other than that it was expressible as a function of the truncated trajectory $\tau_{0:n}$; this analysis is therefore readily applicable to many other functions of the trajectory beyond $n$-step returns, as we will see in Section~\ref{sec:framework}. Secondly, note that within the proof we showed that the familiar product of ratios of \emph{action} probabilities \eqref{eq:actionprobs} is precisely equal to the ratio of \emph{trajectory} probabilities \eqref{eq:trajprobs}, a fact we will use in the remainder of the paper.

\subsection{Per-decision importance sampling}\label{sec:pdis}

Whilst the OIS target of Expression~\eqref{eq:ois} is straightforwardly understood, it often has very high variance. A popular variant that aims to address this shortcoming is given by the \emph{per-decision importance sampling} (PDIS) \citep{precup2000} target:
\begin{align}\label{eq:pdis}
     \PDISreturnE{\pi}{\mu}{0:n}
     = \sum_{t=0}^{n-1} \rho^{\pi,\mu}_{1:t} \gamma^t R_t + \rho^{\pi,\mu}_{1:n-1} \gamma^n V(X_n; \pi) \, ,
\end{align}
The intuition behind this estimator is that each individual reward is only weighted by importance ratios for actions that preceded the reward, it being unnecessary to account for the off-policyness of future actions. This estimator is also unbiased, and is described in the literature as often having lower variance than the OIS estimator.
We show below that each constituent term of the PDIS estimator \emph{is} lower variance than the counterpart term in the OIS estimator.

\begin{restatable}{proposition}{PropPDISVar}\label{prop:PDISvar}
    Assuming the support condition \eqref{eq:sc}, each term in the PDIS estimator has variance at most that of the corresponding term in the OIS estimator. That is, for all $0 \leq t \leq n-1$,
    \begin{align*}
        \mathrm{Var}_{\trajectory{\mu}|_{(x, a)}}\!\left(  \rho^{\pi,\mu}_{1:t} \gamma^t R_t   \right) \leq \mathrm{Var}_{\trajectory{\mu}|_{(x, a)}}\!\left(  \rho^{\pi,\mu}_{1:n-1} \gamma^t R_t   \right) \, .
    \end{align*}
\end{restatable}
The proof technique provides the main insight giving rise to the \emph{\ciw} framework described in the next section, so we provide a sketch below.
The fundamental idea is to show that each term in the estimator $\PDISreturnE{\pi}{\mu}{0:n}$ can be viewed as a \emph{conditional expectation} of a corresponding term in the estimator $\ISreturnE{\pi}{\mu}{0:n}$; we can then use the following well known variance decomposition for any two real-valued random variables $Z_1$ and $Z_2$ with finite second moments:
\begin{align}
    \mathrm{Var}(Z_1) = & \mathrm{Var}(\mathbb{E}\left\lbrack Z_1 | Z_2 \right\rbrack) + \mathbb{E}\left\lbrack \mathrm{Var}(Z_1|Z_2) \right\rbrack \nonumber \\
    \geq & \mathrm{Var}(\mathbb{E}\left\lbrack Z_1 | Z_2 \right\rbrack) \label{eq:condvar} \, ,
\end{align}
with the inequality strict whenever $Z_1$ is not $\sigma(Z_2)$-measurable, or not a function of $Z_2$, using non-measure-theoretic terminology. This idea is closely related to the notion of Rao-Blackwellisation, a variance reduction technique which is ubiquitous across statistics and signal processing \citep{casella2002statistical,sarkka2013bayesian,robert2013monte}.

To apply this result to prove Proposition~\ref{prop:PDISvar}, consider the term $ \rho^{\pi,\mu}_{1:n-1} \gamma^t R_t$ from the OIS estimator, and the term $\rho^{\pi,\mu}_{1:t} \gamma^t R_t$ from the PDIS estimator. A direct calculation yields 
\begin{align*}
    &\mathbb{E}_{\trajectory{\mu}|_{(x, a)}}\!\left\lbrack \rho^{\pi,\mu}_{1:n-1} \gamma^t R_t \middle| X_{0:t}, A_{0:t}, R_{t} \right\rbrack \\
    = & \rho^{\pi,\mu}_{1:t} \gamma^t R_t\ \mathbb{E}_{\trajectory{\mu}}\!\left\lbrack \rho^{\pi,\mu}_{t+1:n-1} \middle| X_{0:t}, A_{0:t}, R_{t} \right\rbrack \\
    = & \rho^{\pi,\mu}_{1:t} \gamma^t R_t \, .
\end{align*}
The final equality follows from the general fact that when the support condition \eqref{eq:sc} is satisfied, the expectation of an importance weight with respect to the importance sampling distribution is $1$. 
Thus, the PDIS term really is a conditional expectation of the corresponding term in the OIS estimator. The bootstrap terms in the PDIS and OIS estimators are in fact equal, and hence the result of Proposition~\ref{prop:PDISvar} follows. Note that \cite{liu2019understanding} also analyse the covariance terms, showing that it is possible for high covariances to outweigh the benefits of smaller per-term variance.

We are now ready to generalise the reasoning presented in this section, and present the main conceptual framework of the paper.

\section{Conditional importance sampling: Theory}\label{sec:framework}

The proof of Proposition~\ref{prop:PDISvar} highlights an important observation; the PDIS estimator in Expression~\eqref{eq:pdis} can be interpreted as taking particular conditional expectations of the OIS estimator in Expression~\eqref{eq:ois} as a means of reducing variance. It will turn out that this process of taking conditional expectations is a productive way of both discovering new off-policy importance sampling methods, and also understanding their statistical properties. For this reason, we take some time to spell out this logic more generally.

Consider the problem of estimating $\mathbb{E}_{\partialtrajectory{\pi}{0:n}|_{(x, a)}}\!\left\lbrack \Psi(\tau_{0:n}) \right\rbrack$, for some function $\Psi$ of the truncated trajectory $\tau_{0:n}$, via importance sampling. A standard importance estimator, taking $\tau_{0:n}\sim \partialtrajectory{\mu}{0:n}|_{(x, a)}$, is given by
\begin{align}\label{eq:general-ois}
    \frac{\partialtrajectory{\pi}{0:n}|_{(x, a)}(\tau_{0:n})}{\partialtrajectory{\mu}{0:n}|_{(x, a)}(\tau_{0:n})} \Psi(\tau_{0:n}) \, .
\end{align}
If $\Psi$ extracts an $n$-step return from the trajectory, this yields the standard OIS estimator, and if $\Psi$ extracts a single reward $R_t$, this yields an individual term from the OIS estimator. In Section~\ref{sec:isinrl}, we saw that in this latter case, a way of reducing the variance of the resulting estimator is to take the conditional expectation given the random variables $(X_{0:t}, A_{0:t}, R_t)$, essentially because $\Psi(\tau_{0:n})=R_t$ is expressible as a function of $(X_{0:t}, A_{0:t}, R_t)$, and the trajectory importance weight is \emph{not} expressible as a function of $(X_{0:t}, A_{0:t}, R_t)$, allowing some extraneous sources of noise to be integrated out. We now formalise this in greater generality.

\begin{definition}
    Given a functional $\Psi$ of a trajectory $\tau_{0:n}$, we say that $\Psi$ \emph{factors through} another functional $\Phi$ if there exists a third function $h$ (independent of the MDP) with $\Psi = h \circ \Phi$, or equivalently, if $\Psi(\tau_{0:n})$ can be written as a function of $\Phi(\tau_{0:n})$ for all values of $\tau_{0:n}$. We say that $\Phi$ is a \emph{\scflong} (SCF) for $\Psi$.
\end{definition}

This notion of \scfslong suggests the following general framework for constructing off-policy estimators, generalising the perspective of PDIS given in the previous section.

\vspace{2mm}
\mdfsetup{%
backgroundcolor=black!10,
roundcorner=10pt}
\begin{mdframed}
\textbf{Conditional importance sampling.}\\
Given a target functional $\Psi(\tau_{0:n})$, select an \SCF $\Phi$ for $\Psi$ and construct the estimator
\begin{align}\label{eq:conditional-is}
    \mathbb{E}_{\trajectory{\mu}|_{(x, a)}}\!\!\left\lbrack \frac{\partialtrajectory{\pi}{0:n}|_{(x, a)}(\tau_{0:n})}{\partialtrajectory{\mu}{0:n}|_{(x, a)}(\tau_{0:n})} \middle| \Phi(\tau_{0:n}) \right\rbrack\! \Psi(\tau_{0:n}) \, .
\end{align}
\end{mdframed}

Through different choices of $\Psi$ and $\Phi$, this yields a wide space of possible off-policy learning algorithms; we refer to this as the \emph{\ciw} (\CIW) framework. We begin with some basic analysis of the properties of these estimators.

\begin{restatable}{proposition}{PropConditionalISVar}\label{prop:ConditionalISVar}
    Assume the support condition \eqref{eq:sc} holds. 
    Given a trajectory functional $\Psi$ and an associated \SCF $\Phi$, the estimator in Expression~\eqref{eq:conditional-is} is unbiased for $\mathbb{E}_{\trajectory{\pi}}[\Psi(\tau_{0:n})]$.
    Further, its variance is no greater than that of the OIS estimator in Expression~\eqref{eq:general-ois}.
\end{restatable}

Having established our framework and some basic properties of the associated estimators, we now provide several examples to aid intuition.

\textbf{Examples:}
\begin{itemize}[leftmargin=0.5cm,topsep=-4pt,itemsep=0pt,partopsep=0pt,parsep=0pt]
\item By taking $\Psi(\tau_{0:n}) = \bar{G}^\pi_{0:n} $, $\Phi(\tau_{0:n}) = \tau_{0:n}$ we recover the usual OIS estimator.
\item By taking $\Psi(\tau_{0:n}) = R_t$, and $\Phi(\tau_{0:n}) = (X_{0:t}, A_{0:t}, R_t)$, we recover the terms of the PDIS estimator, as described in Section~\ref{sec:pdis}.
\item By taking $\Psi(\tau_{0:n}) = R_t$, and $\Phi(\tau_{0:n}) = (X_t, A_t, R_t)$, we recover terms closely related to the marginalised importance sampling estimator of \citet{MIS}.
\end{itemize}

\subsection{Orderings and optimality}\label{sec:orderoptimal}

Given the wide space of possible \SCFs $\Phi$ for a given target $\Psi$ encompassed by the \CIW estimators in Expression~\eqref{eq:conditional-is}, we now turn our attention to understanding the statistical properties of these estimators.\footnote{It is possible to get a slightly more streamlined analysis by working with sigma-algebras, rather than functions of the random trajectory. We restrict the exposition in the main paper to the functional perspective for accessibility and simplicity, but provide a measure-theoretic perspective in Appendix~\ref{sec:sigmaalgebras}.}

There is a natural preorder $\precsim$ on \SCFs for a given target $\Psi$,
that specifies that for two such conditioners $\Phi_1$ and $\Phi_2$, we have $\Phi_1 \precsim \Phi_2$ if there exists a function $h$ such that $\Phi_1 = h \circ \Phi_2$. The relation $\Phi_1 \precsim \Phi_2$  thus makes rigorous the notion \emph{``all information encoded about the trajectory $\tau_{0:n}$ by $\Phi_1(\tau_{0:n})$ is also encoded by $\Phi_2(\tau_{0:n})$''}.

A second preorder that is particularly relevant to studying the statistical properties of off-policy estimators is that of having lower variance, denoted $\precsim_{\mathbb{V}}$. That is, $\Phi_1 \precsim_\mathbb{V} \Phi_2$ if $\text{Var}(\mathbb{E}\left\lbrack \rho^{\pi,\mu}_{1:n-1} \middle| \Phi_1(\tau_{0:n}) \right\rbrack) \leq \text{Var}(\mathbb{E}\left\lbrack \rho^{\pi,\mu}_{1:n-1} \middle| \Phi_2(\tau_{0:n}) \right\rbrack)$. Note that whilst the preorder $\precsim$ is invariant to the MDP and policies $\pi$ and $\mu$ in question, the variance preorder $\precsim_{\mathbb{V}}$ is not. This potentially complicates our variance analysis; however, the following proposition establishes a useful relationship between these two preorders.

\begin{restatable}{proposition}{PropRefine}\label{prop:refine}
    For any given MDP, and pair of policies $\pi$ and $\mu$ satisfying \eqref{eq:sc}, and target functional $\Psi$, the variance preorder \emph{refines} the inclusion preorder. That is, for any two \SCFs $\Phi_1$, $\Phi_2$ of $\Psi$, if $\Phi_1 \precsim \Phi_2$, then we have $\Phi_1 \precsim_\mathbb{V} \Phi_2$.
\end{restatable}

The connection established in Proposition~\ref{prop:refine} will allow us to address the question of optimality: which \SCFs for $\Psi$ yield the lowest variance estimator given in Expression~\eqref{eq:conditional-is}?

\begin{restatable}{proposition}{PropOptimalConditioner}\label{prop:optimalconditioner}
    An \SCF for $\Psi$ for which the associated estimator in Expression~\eqref{eq:conditional-is} achieves minimal variance is $\Psi$ itself.
\end{restatable}

This result gives guidance for choosing a conditioner $\Phi$ for a given target $\Psi$; we study several such algorithms in more detail in Section~\ref{sec:algorithms}.

\subsection{Beyond sufficient conditioning functionals}\label{sec:beyondvalid}

So far, we have enforced the condition that if $\Psi$ is a target functional, a conditioner $\Phi$ used to form the conditional importance-weighted term in Expression~\eqref{eq:conditional-is} should be such that $\Psi(\tau_{0:n})$ is expressible in terms of $\Phi(\tau_{0:n})$. This condition ensures that the resulting estimator is unbiased, as shown in Proposition~\ref{prop:ConditionalISVar}. However, relaxing this condition gives an even greater collection of off-policy estimators. Such estimators formed with functionals $\Phi$ which are \emph{not} \SCFs for $\Psi$ will generally be biased, but in many circumstances may be particularly low-variance, allowing for a bias-variance trade-off to be made.

\textbf{Example:}
\begin{itemize}[leftmargin=0.5cm,topsep=-4pt,itemsep=0pt,partopsep=0pt,parsep=0pt]
\item By taking $\Psi(\tau_{0:n}) = \sum_{t=0}^{n-1} \gamma^t R_t + \gamma^n V(X_n;\pi)$, and $\Phi(\tau_{0:n}) = 0$ (i.e., a function independent of the trajectory), we recover $n$-step uncorrected returns, popularly used in deep reinforcement learning.
\end{itemize}

\subsection{Bias, variance, and estimation difficulty}

We now discuss the various trade-offs inherent within the choice of $\Phi$ required by the \CIW framework.
\begin{figure}
    \centering
    \includegraphics[keepaspectratio,width=.3\textwidth]{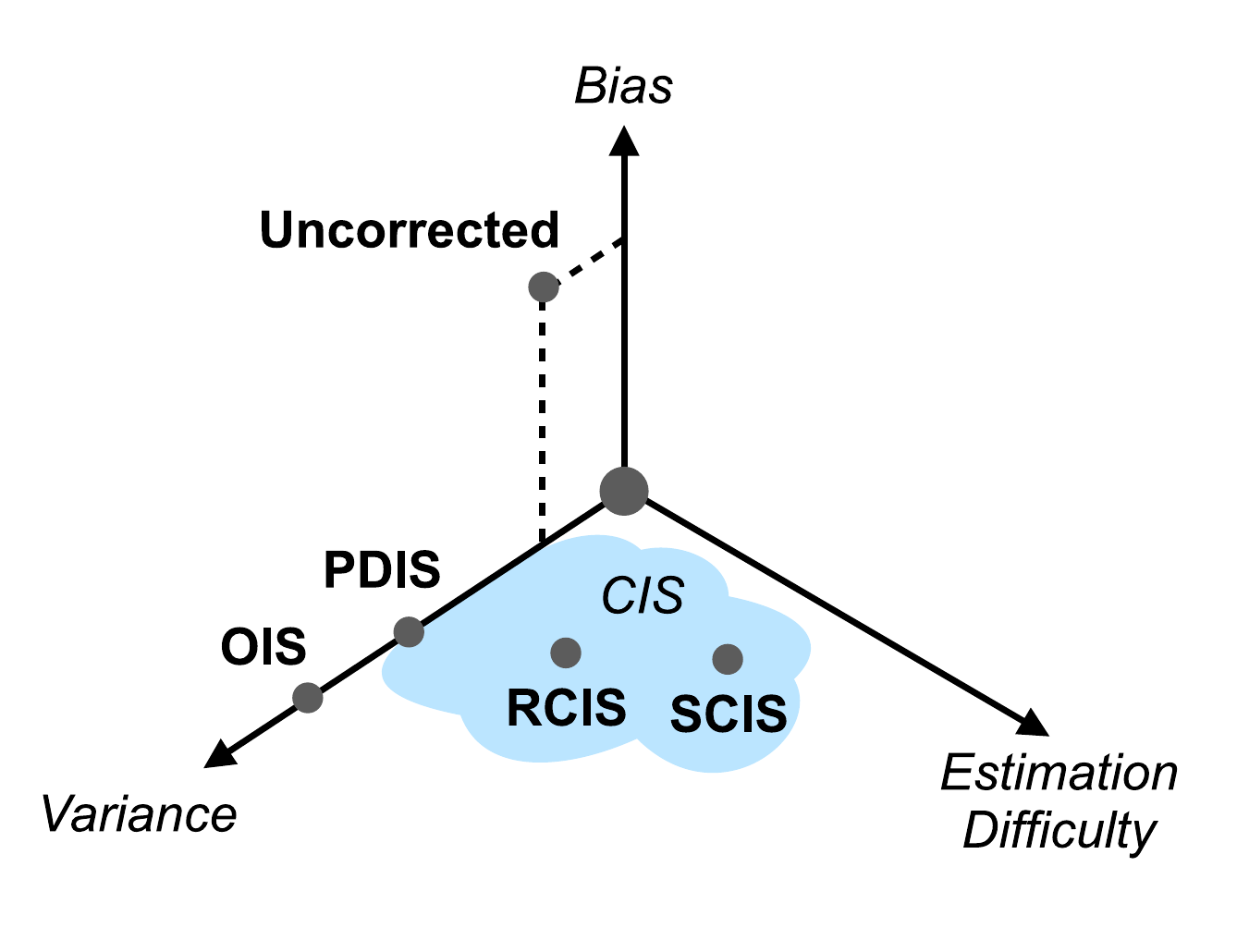}
    \caption{Schematic illustration of three traded-off quantities associated with \CIW estimators.}
    \label{fig:tradeoff}
\end{figure}
Proposition~\ref{prop:ConditionalISVar} shows that any $\Phi$ that is an \SCF for $\Psi$ yields an unbiased off-policy estimator. As described in Section~\ref{sec:beyondvalid}, choosing $\Phi$ which is \emph{not} an \SCF for $\Psi$ generally results in the introduction of bias, but may also offer a further substantial reduction in variance. In addition, there is the question of whether for a given $\Phi$, the importance weight above is available analytically (as in the case of per-decision importance sampling, for example), or whether the weight itself must be estimated, as is the case for several concrete \CIW algorithms, RCIS and SCIS, which we describe in Section~\ref{sec:algorithms}. Figure~\ref{fig:tradeoff} schematically illustrates the trade-offs between these three quantities made by several algorithms in the \CIW framework.

\section{Conditional importance sampling: Algorithms}\label{sec:algorithms}

Having set out the \CIW framework, we now investigate several novel algorithms which naturally arise from it.

\subsection{Return-conditioned importance sampling}

Consider taking the $n$-step truncated return as our target: $\Psi(\tau_{0:n}) = \sum_{t=0}^{n-1}\gamma^t R_t$, and following the optimality result of Proposition~\ref{prop:optimalconditioner}, taking the conditioner $\Phi = \Psi$ to be this return too. This yields a conditional importance weight of the form
\begin{align*}
    \mathbb{E}_{\trajectory{\mu}|_{(x,a)}}\!\left\lbrack \rho^{\pi,\mu}_{1:n-1} \middle| \sum_{t=0}^{n-1}\gamma^t R_t \right\rbrack \, .
\end{align*}
It is possible to express this conditional importance weight more directly, as the following result shows.
\begin{restatable}{proposition}{PropRDIS}\label{prop:rdis}
    Assume the support condition \eqref{eq:sc}. 
    For a given policy $\mu$ let $p^\mu|_{(x, a)}$ be the probability mass function of $\sum_{t=0}^{n-1}\gamma^t R_t$ under $\trajectory{\mu}|_{(x, a)}$. Then we have
    \begin{align}\label{eq:rciw}
        \mathbb{E}_{\trajectory{\mu}|_{(x,a)}}\!\left\lbrack \rho^{\pi,\mu}_{1:n-1} \middle| \sum_{t=0}^{n-1} \gamma^t R_t \right\rbrack = \frac{p^\pi|_{(x, a)}(\sum_{t=0}^{n-1}\gamma^t R_t)}{p^\mu|_{(x, a)}(\sum_{t=0}^{n-1}\gamma^t R_t)} \, .
    \end{align}
\end{restatable}
That is, the optimal conditional importance weight for the $n$-step bootstrapped return is the ratio of the probabilities of the returns themselves under the target and behaviour distributions. This is appealing since it shifts the focus from (potentially irrelevant) policy probabilities directly to probabilities of generating a certain return value.
Due to this property, we term the corresponding estimator the \emph{return-conditioned importance sampling} (RCIS) estimator, given by:
\begin{align*}
    \mathbb{E}_{\trajectory{\mu}|_{(x,a)}}\!\left\lbrack \rho^{\pi,\mu}_{1:n-1} \middle| G \right\rbrack G + \rho^{\pi,\mu}_{1:n-1} \gamma^n V(X_n;\pi) \, ,
\end{align*}
where $G=\sum_{t=0}^{n-1} \gamma^t R_t$.
We note that several further variations of return-conditioned importance sampling are available, such as using an importance weight conditioned on the entire bootstrapped return.

There are strong connections here to distributional reinforcement learning \citep{morimura2010nonparametric,bellemare2017distributional,dabney2018distributional}, in which approximations to return distributions are learnt directly through interaction with the environment.

\subsection{Reward-conditioned and state-conditioned importance sampling}

The previous section establishes return-conditioned importance sampling as the optimal (with respect to estimator variance) unbiased means of importance weighting an entire return. However, this leaves open the question as to whether improvements can be made by importance weighting the individual terms of a return separately, as in per-decision importance sampling. If we interpret each reward $R_t$ in the return $\bar{G}^\pi_{0:n}$ as a target in its own right, Proposition~\ref{prop:optimalconditioner} shows that the corresponding optimal unbiased importance weight is
\begin{align*}
    \mathbb{E}_{\trajectory{\mu}|_{(x, a)}}\!\left\lbrack \frac{\partialtrajectory{\pi}{0:n}|_{(x, a)}(\tau_{0:n})}{\partialtrajectory{\mu}{0:n}|_{(x, a)}(\tau_{0:n})} \middle| R_t \right\rbrack \, .
\end{align*}
We refer to the use of these weights as \emph{reward-conditioned importance sampling}. 
Another estimator of interest that we mention due to its connections with existing off-policy evaluation algorithms and model-based reinforcement learning is given by (suboptimally) conditioning on the tuple $(X_t, A_t, R_t)$ instead of $R_t$ itself. In this case, we obtain the importance weight
\begin{align}\label{eq:scis}
    \mathbb{E}_{\trajectory{\mu}|_{(x, a)}}\!\left\lbrack \frac{\partialtrajectory{\pi}{0:n}|_{(x, a)}(\tau_{0:n})}{\partialtrajectory{\mu}{0:n}|_{(x, a)}(\tau_{0:n})} \middle| X_t, A_t, R_t \right\rbrack \, ,
\end{align}
which can be shown (see Appendix~\ref{sec:proofs}) to be equal to
\begin{align}\label{eq:sciw}
    \frac{p^\pi_{t}|_{(x, a)}(X_t)}{p^\mu_{t}|_{(x, a)}(X_t)} \times \frac{\pi(A_t|X_t)}{\mu(A_t|X_t)} \, ,
\end{align}
where $p^\pi_{t}|_{(x, a)}$ represents the distribution over the state at time $t$ starting at state-action pair $(x, a)$ and following $\pi$ thereafter. Thus, learning this conditional importance weight is closely related to learning the difference between the two transition models $p^\mu_{t}$ and $p^\pi_{t}$. For this reason, we refer to the use of the importance weight in Expression~\eqref{eq:scis} as \emph{state-conditioned importance sampling} (SCIS). There are close ties with the state distribution estimation methods mentioned earlier, such as marginalised importance sampling \citep{MIS,MIS2}, which estimates a similar quantity, but by focusing on learning these transitions distributions separately, rather than their ratio directly, as well as the work of \citet{liu2018breaking}, which learns a ratio of related distributions via a Bellman equation.

\subsection{Importance weight regression}\label{sec:regression}

A crucial practical question about the conditional importance weights appearing in Equations~\eqref{eq:rciw} and \eqref{eq:sciw} (and indeed in the general \CIW estimator in Equation~\eqref{eq:conditional-is}), is how these should be estimated when they are not available analytically. A general approach is given by solving the following regression problem:
\begin{align}\label{eq:regression}
    \min_{\theta} \mathbb{E}_{\trajectory{\mu}|_{(x,a)}}\!\!\!\left\lbrack\! \left(f_\theta(\Phi(\tau_{0:n})) - \frac{\partialtrajectory{\pi}{0:n}|_{(x, a)}(\tau_{0:n})}{\partialtrajectory{\mu}{0:n}|_{(x, a)}(\tau_{0:n})}\right)^{\!2} \right\rbrack  .
\end{align}
In words, we attempt to predict the trajectory importance weight via the function $f_\theta$ parameterised by $\theta$, using solely the information contained in $\Phi(\tau_{0:n})$. In addition, a single regressor could be used across all initial state-action pairs, taking these quantities as additional input (i.e., $f_{\theta}(x, a, \Phi(\tau_{0:n}))$), and thus allowing for generalisation across actions and states.
In practice, global minimisation of this objective will likely not be possible, and it may be desirable to modify the objective to take into account the magnitude of the target term $\Psi(\tau_{0:n})$ (e.g. the $n$-step return) to reduce the variance of the resulting approximate solution, for example. One such modified objective takes the form
\begin{align}\label{eq:regression+}
    \min_{\theta} \mathbb{E}_{\trajectory{\mu}|_{(x,a)}}\!\!\!\left\lbrack\! \left(\!\left(\!f_\theta(\Phi(\tau_{0:n}))\! -\! \frac{\partialtrajectory{\pi}{0:n}|_{(x, a)}(\tau_{0:n})}{\partialtrajectory{\mu}{0:n}|_{(x, a)}(\tau_{0:n})}\!\right) \!\!\Psi(\tau_{0:n})\!\! \right)^{\!\!2} \right\rbrack \! .
\end{align}
The following result grounds these objectives.
\begin{restatable}{proposition}{PropRegressionMin}
    A global minimum for each of the objectives in Expressions~\eqref{eq:regression} and \eqref{eq:regression+} is given by
    \begin{align*}
        f_\theta(\Phi(\tau_{0:n}))
        =
        \mathbb{E}_{\trajectory{\mu}|_{(x, a)}}\!\left\lbrack
            \frac{\partialtrajectory{\pi}{0:n}|_{(x, a)}(\tau_{0:n})}{\partialtrajectory{\mu}{0:n}|_{(x, a)}(\tau_{0:n})}
            \middle|
            \Phi(\tau_{0:n})
            \right\rbrack \, .
    \end{align*}
\end{restatable}

\section{Experiments}\label{sec:experiments}

\begin{figure*}[!ht]
    \centering
    \includegraphics[keepaspectratio,width=.88\textwidth]{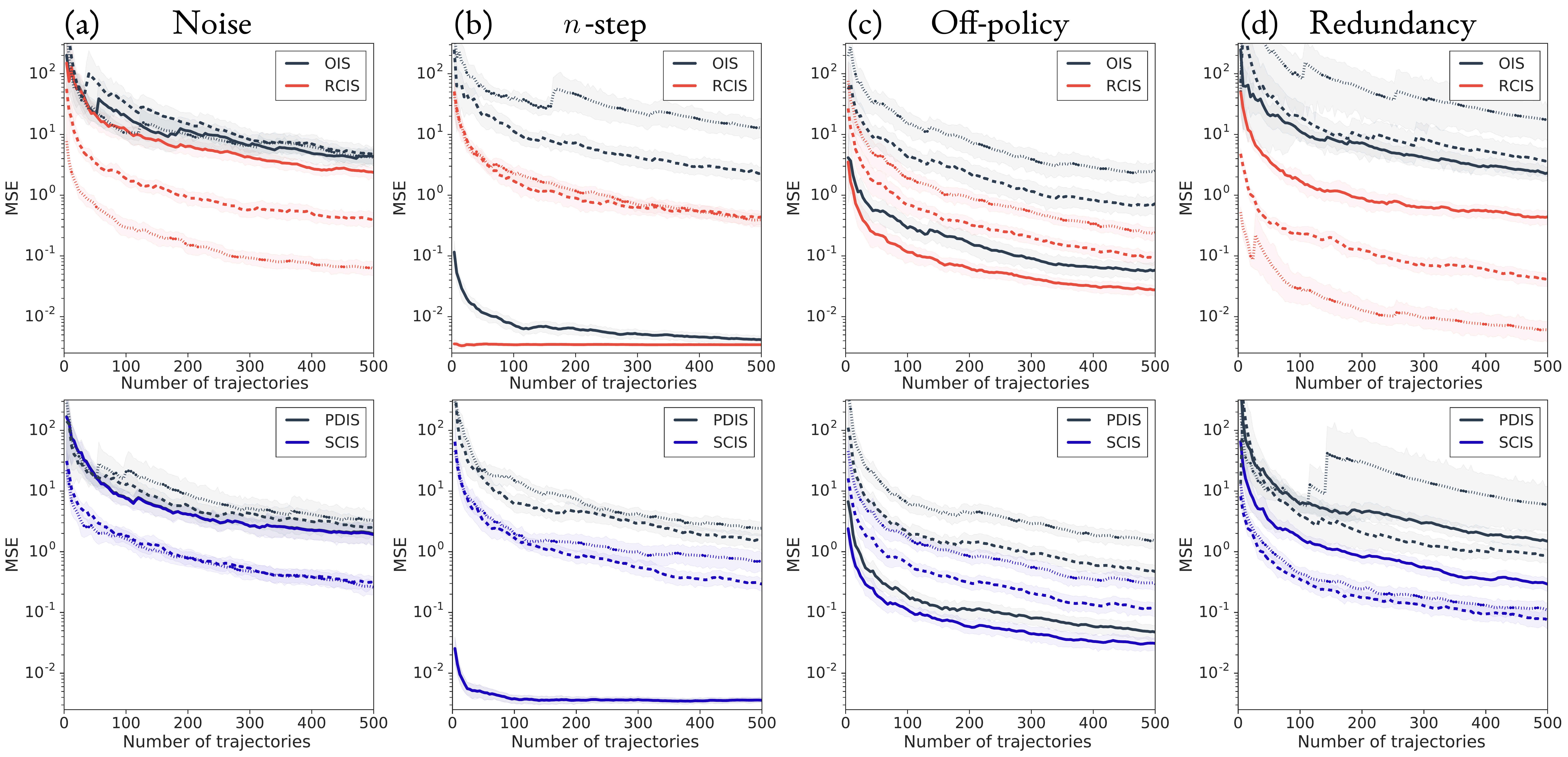}
    \vspace{-0.2cm}
    \caption{Operator estimation MSE as a function of sample number for OIS, PDIS, RCIS, and SCIS, on a chain MDP with varying (\textbf{a}) levels of transition noise, (\textbf{b}) $n$-step updates, (\textbf{c}) separation of policies, and (\textbf{d}) redundancy in action sets, as outlined in Table~\ref{tab:params}. Shaded regions indicate bootstrapped 95\% confidence intervals.}
    \label{fig:operatorestimation}
\end{figure*}

To complement the \CIW framework and the theoretical analysis conducted in earlier sections, we provide several simple illustrative experiments that demonstrate (i) that \CIW algorithms can deliver substantial variance reduction, and (ii) that the regression approach of Section~\ref{sec:regression} can be used to obtain practical implementations of \CIW algorithms.
We exhibit results on a classic chain environment, with both tabular and linear function approximation methods; full experiment specifications are given in Appendix~\ref{sec:appendix:experiments}.

\subsection{Operator estimation}

We begin with the task of off-policy estimation of the application of the $n$-step Bellman operator $(T^\pi)^n$ to a fixed Q-function via trajectories generated by following the behaviour policy $\mu$. This serves as a precursor for off-policy evaluation, and allows us to disentangle the variance reduction achieved by conditional importance sampling from compounding bootstrapping effects.

Results are shown for a chain environment in Figure~\ref{fig:operatorestimation}. We plot MSE for both OIS and PDIS, as well as conditional importance sampling versions of these algorithms, RCIS and SCIS, with the conditional importance weights provided by a pre-computed \emph{oracle}. The use of an oracle allows us to separate the variance reduction effects of conditional importance sampling from the potential errors introduced by the regression approach described in Section~\ref{sec:regression}. In each of the four sub-plots of Figure~\ref{fig:operatorestimation}, we vary one property of the estimation problem, to illustrate how performance of the methods under study changes. In all cases, we plot results for three different settings of the parameter in question, with solid lines corresponding to \emph{low} values of this parameter, and finely-dashed lines corresponding to high values of the parameter; see Table~\ref{tab:params}. ``Noise'' refers to the transition noise added to the chain, $\beta$ controls mismatch between the target $\pi$ and behaviour $\mu$ policies, by replacing the target with a mixture $\beta\pi+(1-\beta)\mu$, and ``extra actions'' describes how many extra (redundant) copies of each action are added to the environment.

\begin{table}[!h]
    \centering
    \captionsetup{justification=centering}
    \caption{Parameter values for Figure~\ref{fig:operatorestimation}.}
    \vspace{-0.35cm}
    \label{tab:params}
    \begin{tabular}{ccccc}
    \toprule
        Line type & Noise & $n$ & $\beta$ & Extra actions \\ \midrule
        Solid & 0\% & 2 & 0.1 & 0 \\ 
        Dashed & 10\% & 4 & 0.5 & 1 \\ 
        Finely-dashed & 50\% & 7 & 1.0 & 3\\ \bottomrule
    \end{tabular}
\end{table}

In all cases, the \CIW methods outperform their existing counterparts, with more pronounced improvements in the presence of larger $n$, more transition noise, greater off-policyness, and high level of action redundancy. 

\subsection{Policy evaluation}\label{sec:policy-evaluation}

\begin{figure*}[!ht]
    \centering
    \includegraphics[keepaspectratio,width=.88\textwidth]{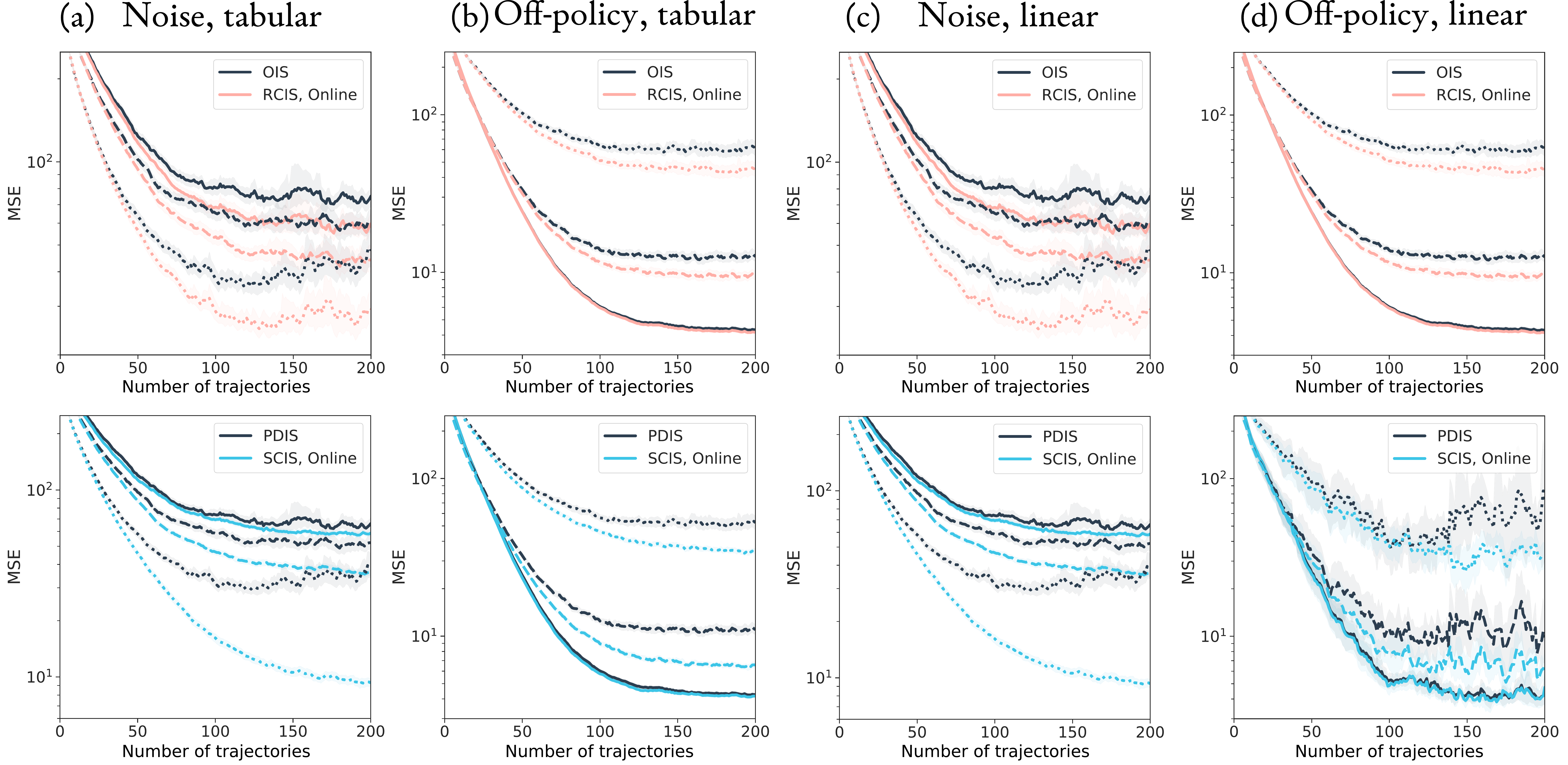}
    \vspace{-0.2cm}
    \caption{Policy evaluation MSE as a function of number of trajectories for OIS, RCIS, PDIS, and SCIS, with both tabular and function approximation variants. Shaded regions indicate bootstrapped 95\% confidence intervals.}
    \label{fig:policyevaluation}
    \vspace{-0.4cm}
\end{figure*}

We now consider the full task of off-policy policy evaluation using $n$-step returns along trajectories generated by a behaviour policy, with importance weights provided by existing and new \CIW algorithms. We report results for the same chain environment
as for the operator estimation experiments in Figure~\ref{fig:policyevaluation}, with varying levels of transition noise and off-policyness as described in Table~\ref{tab:params}. We give results for online variants of \CIW algorithms by solving the empirical version of Expression~\eqref{eq:regression} (based on the observed trajectories) exactly for each different value of the functional observed; complete results including the oracle versions of the \CIW algorithms are given in Appendix~\ref{sec:appendix:additional-results}.
We show results for tabular evaluation, as well as versions using tile-coding linear function approximation \citep{Sutton2018} (full details in Appendix~\ref{sec:appendix:policy-evaluation-details}). Generally, we observe that the online versions of the \CIW algorithms generally give a noticeable improvement over their non-conditional versions. These results serve as a proof of concept that practical, online versions of the \CIW algorithms introduced in Section~\ref{sec:algorithms} can improve over non-conditional baselines. We expect that with further research into regression methods described in Section~\ref{sec:regression}, the gap between oracle and online \CIW algorithms can be narrowed.

\section{Discussion}

We have unified several existing importance sampling algorithms via a new conceptual framework based on conditional expectations of importance weights, allowing for straightforward analysis and comparison, in addition to the development of new algorithms.

There remain many interesting investigations to be carried out towards theoretically and empirically understanding how the \CIW framework interacts with complementary approaches for variance reduction, such as weighted importance sampling and importance weight truncation. We expect several further directions to prove fruitful for future work, including further exploration of the space of \CIW algorithms, scaling up \CIW algorithms to work in combination with deep RL architectures, and further investigation into relationships between particular \CIW algorithms with other sub-fields of RL (such as RCIS and distributional RL).

\section*{Acknowledgements}

We thank Adam White for detailed feedback on an earlier version of this paper, and the anonymous reviewers for helpful comments during the review process.

\bibliographystyle{abbrvnat}
\bibliography{main}

\newpage
\onecolumn
\appendix

\section*{\centering APPENDICES: Conditional Importance Sampling for Off-Policy Learning}

\section{Proofs}\label{sec:proofs}

\PropConditionalISVar*

\begin{proof}
    The proof of unbiasedness follows the logic of Proposition~\ref{prop:nstepunbiased}'s proof and the proof for the variance upper bound follows the logic of Proposition~\ref{prop:PDISvar}'s proof. Beginning with unbiasedness, we make the following calculation:
    \begin{align*}
        \mathbb{E}_{\partialtrajectory{\mu}{0:n}|_{(x, a)}}\!\left\lbrack \mathbb{E}_{\trajectory{\mu}|_{(x, a)}}\!\left\lbrack \frac{\partialtrajectory{\pi}{0:n}|_{(x, a)}(\tau_{0:n})}{\partialtrajectory{\mu}{0:n}|_{(x, a)}(\tau_{0:n})} \middle| \Phi(\tau_{0:n}) \right\rbrack \Psi(\tau_{0:n}) \right\rbrack & \overset{(a)}{=}
        \mathbb{E}_{\partialtrajectory{\mu}{0:n}|_{(x, a)}}\!\left\lbrack \mathbb{E}_{\trajectory{\mu}|_{(x, a)}}\!\left\lbrack \frac{\partialtrajectory{\pi}{0:n}|_{(x, a)}(\tau_{0:n})}{\partialtrajectory{\mu}{0:n}|_{(x, a)}(\tau_{0:n})}  \Psi(\tau_{0:n}) \middle| \Phi(\tau_{0:n}) \right\rbrack \right\rbrack \\
        &\overset{(b)}{=} \mathbb{E}_{\partialtrajectory{\mu}{0:n}|_{(x, a)}}\!\left\lbrack \frac{\partialtrajectory{\pi}{0:n}|_{(x, a)}(\tau_{0:n})}{\partialtrajectory{\mu}{0:n}|_{(x, a)}(\tau_{0:n})}  \Psi(\tau_{0:n})  \right\rbrack \\
        & \overset{(c)}{=} \mathbb{E}_{\partialtrajectory{\pi}{0:n}|_{(x, a)}}\!\left\lbrack \Psi(\tau_{0:n})  \right\rbrack \, ,
    \end{align*}
    where $(a)$ follows since $\Phi$ is an \SCF for $\Psi$ (and hence $\Psi(\tau_{0:n})$ is fully determined by $\Phi(\tau_{0:n})$), (b) follows from the tower law of conditional expectations, and (c) follows from standard importance sampling theory. 
    
    For the variance result, we observe that
    \begin{align*}
        \mathbb{E}_{\partialtrajectory{\mu}{0:n}|_{(x, a)}}\!\left\lbrack \frac{\partialtrajectory{\pi}{0:n}|_{(x, a)}(\tau_{0:n})}{\partialtrajectory{\mu}{0:n}|_{(x, a)}(\tau_{0:n})} \middle| \Phi(\tau_{0:n}) \right\rbrack \Psi(\tau_{0:n})  & =
        \mathbb{E}_{\partialtrajectory{\mu}{0:n}|_{(x, a)}}\!\left\lbrack \frac{\partialtrajectory{\pi}{0:n}|_{(x, a)}(\tau_{0:n})}{\partialtrajectory{\mu}{0:n}|_{(x, a)}(\tau_{0:n})}  \Psi(\tau_{0:n}) \middle| \Phi(\tau_{0:n}) \right\rbrack  \, ,
    \end{align*}
    which follows since $\Phi$ is an \SCF for $\Psi$.
    Therefore, this estimator is a conditional expectation of the OIS estimator
    \begin{align*}
        \frac{\partialtrajectory{\pi}{0:n}|_{(x, a)}(\tau_{0:n})}{\partialtrajectory{\mu}{0:n}|_{(x, a)}(\tau_{0:n})}  \Psi(\tau_{0:n}) \, ,
    \end{align*}
    and therefore the conclusion follows by direct application of Equation~\eqref{eq:condvar} which was used to establish Proposition~\ref{prop:PDISvar}, taking $Z_1 = \frac{\partialtrajectory{\pi}{0:n}|_{(x, a)}(\tau_{0:n})}{\partialtrajectory{\mu}{0:n}|_{(x, a)}(\tau_{0:n})} \Psi(\tau_{0:n})$ and $Z_2 = \Phi(\tau_{0:n})$.
\end{proof}

\PropRefine*

\begin{proof}
    Assume we have $\Phi_1 \precsim \Phi_2$ for two \scfslong $\Phi_1,\Phi_2$ for $\Psi$.
    Since $\Phi_1(\tau_{0:n})$ is a function of $\Phi_2(\tau_{0:n})$, we have that $\mathbb{E}[\rho^{\pi,\mu}_{1:n-1} | \Phi_1(\tau_{0:n})] = \mathbb{E}[\mathbb{E}[\rho^{\pi,\mu}_{1:n-1} | \Phi_2(\tau_{0:n})]|\Phi_1(\tau_{0:n})]$ by the tower property for conditional expectations.
    The statement now follows from the conditional variance formula \eqref{eq:condvar}.
\end{proof}

\PropOptimalConditioner*

\begin{proof}
    This follows by first observing that $\Psi(\tau_{0:n})$ is a minimal sufficient conditioning functional for $\Psi$ with respect to the ordering induced by $\precsim$; this is immediate from the definition. Next, since $\precsim_\mathbb{V}$ refines $\precsim$ (by Proposition~\ref{prop:refine}), we have that $\Psi(\tau_{0:n})$ is also a minimal sufficient conditioning functional with respect to $\precsim_{\mathbb{V}}$, and the statement follows.
\end{proof}

\PropRDIS*

\begin{proof}
    As in the discussion in Section~\ref{sec:isinrl}, we have
    \begin{align*}
        \rho^{\pi,\mu}_{1:n-1}
        = \frac{\partialtrajectory{\pi}{0:n}|_{(x,a)}(\tau_{0:n})}{\partialtrajectory{\mu}{0:n}|_{(x,a)}(\tau_{0:n})}
        \, .
    \end{align*}
    We then decompose
    \begin{align*}
        \mathbb{E}_{\partialtrajectory{\mu}{0:n}|_{(x, a)}} \left\lbrack \frac{\partialtrajectory{\pi}{0:n}|_{(x,a)}(\tau_{0:n})}{\partialtrajectory{\mu}{0:n}|_{(x,a)}(\tau_{0:n})} \middle| \sum_{t=0}^{n-1}\gamma^t R_t \right\rbrack
        = & \mathbb{E}_{\partialtrajectory{\mu}{0:n}|_{(x, a)}}\!\left\lbrack \frac{p_{\pi}|_{(x, a)}(\sum_{t=0}^{n-1}\gamma^t R_t)\partialtrajectory{\pi}{0:n}|_{(x,a)}(\tau_{0:n} | \sum_{t=0}^{n-1}\gamma^t R_t)}{p_{\mu}|_{(x, a)}(\sum_{t=0}^{n-1}\gamma^t R_t)\partialtrajectory{\mu}{0:n}|_{(x,a)}(\tau_{0:n}|\sum_{t=0}^{n-1}\gamma^t R_t)} \middle| \sum_{t=0}^{n-1}\gamma^t R_t \right\rbrack \\
        = & \frac{p_{\pi}|_{(x, a)}(\sum_{t=0}^{n-1}\gamma^t R_t)}{p_{\mu}|_{(x, a)}(\sum_{t=0}^{n-1}\gamma^t R_t)} \mathbb{E}_{\partialtrajectory{\mu}{0:n}|_{(x, a)}}\!\left\lbrack \frac{\partialtrajectory{\pi}{0:n}|_{(x,a)}(\tau_{0:n} | \sum_{t=0}^{n-1}\gamma^t R_t)}{\partialtrajectory{\mu}{0:n}|_{(x,a)}(\tau_{0:n}|\sum_{t=0}^{n-1}\gamma^t R_t)} \middle| \sum_{t=0}^{n-1}\gamma^t R_t \right\rbrack \\
        = & \frac{p_{\pi}|_{(x, a)}(\sum_{t=0}^{n-1}\gamma^t R_t)}{p_{\mu}|_{(x, a)}(\sum_{t=0}^{n-1}\gamma^t R_t)} \, ,
    \end{align*}
    as required.
\end{proof}

\PropRegressionMin*

\begin{proof}
    We begin by restating Expression~\eqref{eq:regression}, and use the tower law of conditional expectation  as follows:
    \begin{align*}
    & \mathbb{E}_{\trajectory{\mu}|_{(x,a)}}\!\left\lbrack \left(f_\theta(\Phi(\tau_{0:n})) - \frac{\partialtrajectory{\pi}{0:n}|_{(x, a)}(\tau_{0:n})}{\partialtrajectory{\mu}{0:n}|_{(x, a)}(\tau_{0:n})}\right)^2 \right\rbrack  \\
    = & \mathbb{E}_{\trajectory{\mu}|_{(x,a)}}\!\left\lbrack \mathbb{E}_{\trajectory{\mu}|_{(x,a)}}\!\left\lbrack \left(f_\theta(\Phi(\tau_{0:n})) - \frac{\partialtrajectory{\pi}{0:n}|_{(x, a)}(\tau_{0:n})}{\partialtrajectory{\mu}{0:n}|_{(x, a)}(\tau_{0:n})}\right)^2 \middle| \Phi(\tau_{0:n}) \right\rbrack \right\rbrack \, .
    \end{align*}
    The inner conditional expectation is of the form $\mathbb{E}_Y[(z - Y)^2]$; viewed as a function of $z$, it is well known that the minimiser of such an expression is $z = \mathbb{E}[Y]$. Thus, for a fixed value of $\Phi(\tau_{0:n})$, the optimal value of $f_\theta(\Phi(\tau_{0:n}))$ is given by
    \begin{align*}
        \mathbb{E}_{\trajectory{\mu}}\left\lbrack \frac{\partialtrajectory{\pi}{0:n}|_{(x, a)}(\tau_{0:n})}{\partialtrajectory{\mu}{0:n}|_{(x, a)}(\tau_{0:n})} \middle| \Phi(\tau_{0:n}) \right\rbrack \, .
    \end{align*}
    Therefore, the global optimiser of Expression~\eqref{eq:regression} is given precisely by the function
    \begin{align*}
        f_\theta(\Phi(\tau_{0:n})) = \mathbb{E}_{\trajectory{\mu}}\left\lbrack \frac{\partialtrajectory{\pi}{0:n}|_{(x, a)}(\tau_{0:n})}{\partialtrajectory{\mu}{0:n}|_{(x, a)}(\tau_{0:n})} \middle| \Phi(\tau_{0:n}) \right\rbrack \, ,
    \end{align*}
    as required. For Expression~\eqref{eq:regression+}, in a similar manner we can write the following:
    \begin{align*}
        & \mathbb{E}_{\trajectory{\mu}|_{(x,a)}}\!\left\lbrack \left(f_\theta(\Phi(\tau_{0:n})) - \frac{\partialtrajectory{\pi}{0:n}|_{(x, a)}(\tau_{0:n})}{\partialtrajectory{\mu}{0:n}|_{(x, a)}(\tau_{0:n})}\right)^2 \Psi(\tau_{0:n})^2 \right\rbrack  \\
        = & \mathbb{E}_{\trajectory{\mu}|_{(x,a)}}\!\left\lbrack \mathbb{E}_{\trajectory{\mu}|_{(x,a)}}\!\left\lbrack \left(f_\theta(\Phi(\tau_{0:n})) - \frac{\partialtrajectory{\pi}{0:n}|_{(x, a)}(\tau_{0:n})}{\partialtrajectory{\mu}{0:n}|_{(x, a)}(\tau_{0:n})}\right)^2 \middle| \Phi(\tau_{0:n}) \right\rbrack \Psi(\tau_{0:n})^2 \right\rbrack \, ,
    \end{align*}
    with the equality following from the fact that $\Phi$ is a \scflong for $\Psi$. Now we may proceed in an identical manner to that for Expression~\eqref{eq:regression}, and the claim follows.
\end{proof}

We also record a precise result on the form of the SCIS weights described in Section~\ref{sec:algorithms} below.

\begin{proposition}
    As described in Section~\ref{sec:algorithms}, assuming the support condition, we have
    \begin{align*}
        \mathbb{E}_{\trajectory{\mu}|_{(x, a)}}\!\left\lbrack \frac{\partialtrajectory{\pi}{0:n}|_{(x, a)}(\tau_{0:n})}{\partialtrajectory{\mu}{0:n}|_{(x, a)}(\tau_{0:n})} \middle| X_t, A_t, R_t \right\rbrack = \frac{p^\pi_{t}|_{(x, a)}(X_t)}{p^\mu_{t}|_{(x, a)}(X_t)} \times \frac{\pi(A_t|X_t)}{\mu(A_t|X_t)} \, .
    \end{align*}
\end{proposition}
\begin{proof}
    The proof follows by factorising the trajectory probabilities $\partialtrajectory{\pi}{0:n}|_{(x, a)}(\tau_{0:n})$, $\partialtrajectory{\mu}{0:n}|_{(x, a)}(\tau_{0:n})$ in the following manner, using the Markov property of the environment:
    \begin{align*}
        \partialtrajectory{\pi}{0:n}|_{(x, a)}(\tau_{0:n}) = p^\pi_{t}|_{(x, a)}(X_t) \pi(A_t|X_t) \partialtrajectory{\pi}{t:n}|_{(X_t, A_t)}(\tau_{t:n}) \partialtrajectory{\pi}{0:t-1}|_{(x, a)}(\tau_{0:t-1} |X_t) \, ,
    \end{align*}
    where we write $\partialtrajectory{\pi}{0:t-1}|_{(x, a)}(\tau_{0:t-1} |X_t)$ for probability mass associated with the trajectory $\tau_{0:t-1}$ under $\partialtrajectory{\pi}{0:t}$, conditional on the trajectory visiting the state $X_t$ at time $t$. Using conditional independence, we therefore have
    \begin{align*}
        & \mathbb{E}_{\trajectory{\mu}|_{(x, a)}}\!\left\lbrack \frac{\partialtrajectory{\pi}{0:n}|_{(x, a)}(\tau_{0:n})}{\partialtrajectory{\mu}{0:n}|_{(x, a)}(\tau_{0:n})} \middle| X_t, A_t, R_t \right\rbrack \\
        = & \frac{p^\pi_{t}|_{(x, a)}(X_t) \pi(A_t|X_t)}{p^\mu_{t}|_{(x, a)}(X_t) \mu(A_t|X_t)}
        \mathbb{E}_{\trajectory{\mu}|_{(x, a)}}\!\left\lbrack \frac{\partialtrajectory{\pi}{t:n}|_{(X_t, A_t)}(\tau_{t:n}) \partialtrajectory{\pi}{0:t}|_{(x, a)}(\tau_{0:t-1} |X_t)}{\partialtrajectory{\mu}{t:n}|_{(X_t, A_t)}(\tau_{t:n}) \partialtrajectory{\mu}{0:t}|_{(x, a)}(\tau_{0:t-1} |X_t)} \middle| X_t, A_t, R_t \right\rbrack \\
        = & \frac{p^\pi_{t}|_{(x, a)}(X_t) \pi(A_t|X_t)}{p^\mu_{t}|_{(x, a)}(X_t) \mu(A_t|X_t)}
        \mathbb{E}_{\trajectory{\mu}|_{(x, a)}}\!\left\lbrack \frac{\partialtrajectory{\pi}{t:n}|_{(X_t, A_t)}(\tau_{t:n}) }{ \partialtrajectory{\mu}{t:n}|_{(X_t, A_t)}(\tau_{t:n})} \middle| X_t, A_t \right\rbrack
        \mathbb{E}_{\trajectory{\mu}|_{(x, a)}}\!\left\lbrack \frac{\partialtrajectory{\pi}{0:t}|_{(x, a)}(\tau_{0:t-1} |X_t)}{ \partialtrajectory{\mu}{0:t}|_{(x, a)}(\tau_{0:t-1} |X_t)} \middle| X_t \right\rbrack \\
        = & \frac{p^\pi_{t}|_{(x, a)}(X_t) \pi(A_t|X_t)}{p^\mu_{t}|_{(x, a)}(X_t) \mu(A_t|X_t)} \, ,
    \end{align*}
    as required. The final equality follows since both of the conditional expectations are in fact expectations of Radon-Nikodym derivatives under the measure in the ``denominator'' of the derivative, and hence evaluate to $1$ almost surely.
\end{proof}

\section{Experimental details}\label{sec:appendix:experiments}

\subsection{Environment}\label{sec:appendix:envs}

\textbf{Chain.} We use a $6$-state chain environment, with absorbing states at each end of the chain. Two actions, \texttt{left} and \texttt{right}, are available at each state of the chain. Transitions corrupted with $p$\% noise means that with probability $p$, a transition to a uniformly-random adjacent state (independent of the action taken) occurs. Each non-terminal step incurs a reward of $+1$, whilst reaching an absorbing state incurs a one-off reward of $+10$, and the episode then terminates. The initial state of the environment is taken to be the third state from the left. Figure~\ref{fig:chain} provides an illustration.

\begin{figure}[h]
\centering
    \includegraphics[keepaspectratio,width=.65\textwidth]{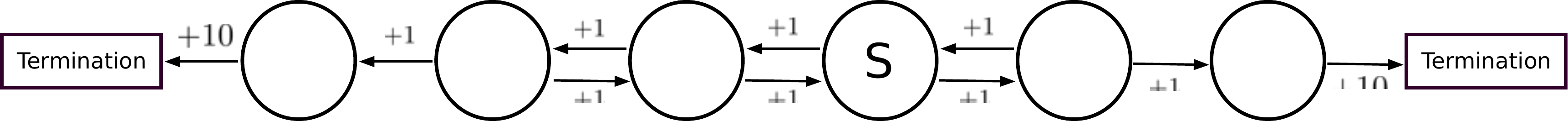}
    \caption{Illustration of the chain environment.}
    \label{fig:chain}
\end{figure}

\subsection{Other experimental details: operator estimation}\label{sec:appendix:operator-estimation-details}

Throughout, the discount factor is taken to be $\gamma = 0.99$, and the Q-function used to form the target $(T^\pi)^n Q$ has its entries sampled independently from the $N(0,0.1)$ distribution. The policies $\pi$ and $\mu$ are drawn independently, with each $\pi(\cdot|x)$ and $\mu(\cdot|x)$ drawn independently from a Dirichlet($1,\ldots,1$) distribution. Default values of parameters are taken as $n=5$, the transition noise level is set to $10\%$, and the learning rate is set to $0.1$, and 100 repetitions of each experiments are performed to compute the bootstrapped confidence intervals.

\subsection{Other experimental details: policy evaluation}\label{sec:appendix:policy-evaluation-details}

The environment and default parameters are exactly the same as in the operator estimation experiments, with the exception that the Q-function is initialised so that all coordinates are $0$, and $n=3$. We estimate bootstrap confidence intervals using 500 repetitions of each experiment. In the linear function approximation experiments, we use a version of tile-coding \citep{Sutton2018}; the specification parametrisation we use is as follows. For a chain of length $K$, we take a weight vector $\mathbf{w}=(w_{k,a} | k \in [K-1], a \in \actionspace) \in \mathbb{R}^{(K-1)\times|\actionspace|}$. Labelling the states of the chain $x_1,\ldots,x_{K}$, we parametrise $Q(x_1, a)$ by $w_{1, a}$, $Q(x_K,a)$ by $w_{{K-1},a}$, and $Q(x_k, a)$ by $\frac{1}{2} w_{{k-1},a} + \frac{1}{2} w_{k, a}$, for each $a \in \actionspace$ and each $1 < k < K$; this is illustrated in Figure~\ref{fig:chainFA}. The weight vector is initialised with all coordinates equal to $0$ in all experiments.

\begin{figure}[h]
    \centering
    \includegraphics[keepaspectratio,width=.65\textwidth]{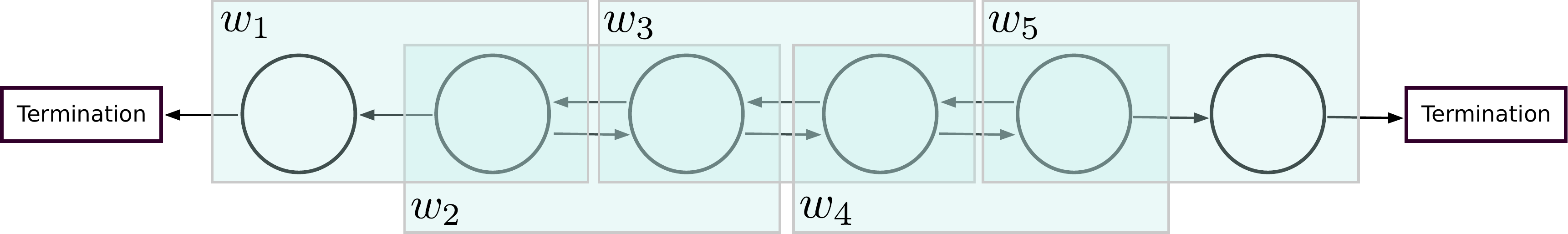}
    \caption{An illustration of the tile-coding scheme used in the linear function approximation scheme; the figure shows how feature weights (for each action) are allocated states. The value prediction at each state is given by averaging the weights allocated to the state.}
    \label{fig:chainFA}
\end{figure}

\subsection{Further experimental results}\label{sec:appendix:additional-results}

In this section, we give in Figure~\ref{fig:policyevaluationappendix} the results described in Section~\ref{sec:policy-evaluation}, including also results for oracle versions of the \CIW algorithms in question. We observe that the performance of the online versions of \CIW algorithms generally closely track that of their oracle counterparts.

\begin{figure}[!h]
    \centering
    \includegraphics[keepaspectratio,width=.9\textwidth]{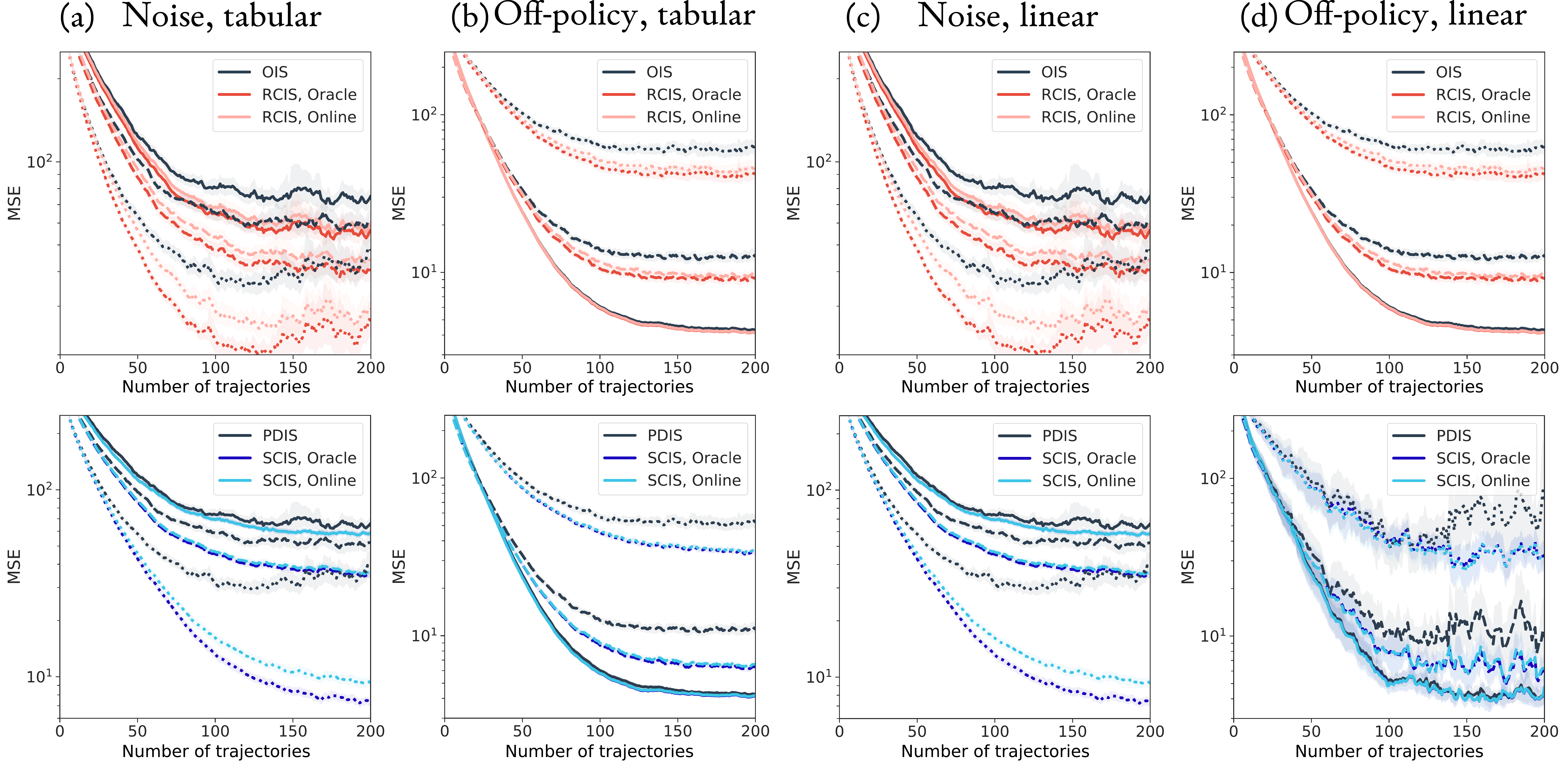}
    \caption{Policy evaluation MSE as a function of number of trajectories for OIS, RCIS, PDIS, and SCIS, with both tabular and function approximation variants. Shaded regions indicate bootstrapped 95\% confidence intervals.}
    \label{fig:policyevaluationappendix}
\end{figure}

\section{Extending the \CIW framework}

\subsection{A measure-theoretic perspective on conditional importance sampling}\label{sec:sigmaalgebras}

In this section, we give a measure-theoretic treatment of the conditional importance sampling framework introduced in Section~\ref{sec:framework} of the main paper. We do not provide any fundamentally new results relative to the main paper, but we believe the measure-theoretic exposition gives a useful perspective, and may be useful for future work.

We begin by returning to the trajectory importance-weighted estimator given in Expression~\eqref{eq:general-ois} in the main paper:
\begin{align*}
    \frac{\partialtrajectory{\pi}{0:n}|_{(x, a)}(\tau_{0:n})}{\partialtrajectory{\mu}{0:n}|_{(x, a)}(\tau_{0:n})} \Psi(\tau_{0:n}) \, .
\end{align*}
This expression weights the target quantity $\Psi(\tau_{0:n})$ by the importance weight associated with the proposal distribution $\partialtrajectory{\mu}{0:n}$ and the target distribution $\partialtrajectory{\pi}{0:n}$. A conditional importance sampling estimator is formed by taking a function $\Phi$ that in the language of the main paper, is a \scflong for $\Psi$, and forming the new estimator
\begin{align*}
    \mathbb{E}\left\lbrack \frac{\partialtrajectory{\pi}{0:n}|_{(x, a)}(\tau_{0:n})}{\partialtrajectory{\mu}{0:n}|_{(x, a)}(\tau_{0:n})} \middle| \Phi(\tau_{0:n}) \right\rbrack\Psi(\tau_{0:n}) \, .
\end{align*}
Proposition~\ref{prop:ConditionalISVar} then shows that the variance of the conditioned estimator is no greater than that of the trajectory-weighted estimator, and, roughly speaking, in many cases it is strictly lower.

Whilst this perspective of conditioning on functionals $\Phi$ of the trajectory is conceptually straightforward and clearly hints at how such techniques can be implemented in practice, as described in Section~\ref{sec:regression}, there are some subtleties introduced by this perspective that make the analysis of the method less straightforward. One such case is illustrated by the following example: consider two \scfslong $\Phi_1$ and $\Phi_2$ for a target $\Psi$, which happen to be related according to the identity $\Phi_1(\tau_{0:n}) = 2\Phi_2(\tau_{0:n})$ for all $\tau_{0:n}$. Intuitively, $\Phi_1$ and $\Phi_2$ encode the same information about $\tau_{0:n}$, and thus the estimators they produce are identical. We might therefore like to be able to treat $\Phi_1$ and $\Phi_2$ as ``identical'' in our analysis, and yet this is made difficult by the focus of the analysis on \emph{functionals} of the trajectory. This is related to the need to work with \emph{preorders} in Section~\ref{sec:orderoptimal}, rather than the perhaps more familiar notion of \emph{partial orders}. One route around this difficulty is to define an equivalence relation over functions of the trajectory, rigorously encoding the notion of ``captures the same information about $\tau_{0:n}$'', and then to work instead with equivalence classes of trajectory functionals under this relation. However, this has the potential to be very unwieldy, and further, it turns out this is essentially equivalent to a much more familiar collection of objects from measure theory, known as sigma-algebras. For formal definitions and background on sigma-algebras, see for example \citet{billingsley}. We note that technically speaking, it is necessary to constrain functionals of the trajectory to be \emph{measurable}; we do not mention this condition further in this section, but return to it in Appendix~\ref{sec:generalising} when describing the application of the conditional importance sampling framework to more general classes of MDPs.
For a general random variable $Z$, we write $\mathscr{F}_Z$ for the sigma-algebra generated by $Z$; in the discussion that follows, all random variables will be defined over the same probability space, which we therefore suppress from the notation in what follows.

The counterpart to a \scflong $\Phi$ is a \emph{\scsalong} (\SCSA) $\mathscr{F}$, which is defined as being a sigma-algebra over the same measurable space as $\mathscr{F}_{\tau_{0:n}}$, with the property that $\mathscr{F}_{\Psi(\tau_{0:n})} \subseteq \mathscr{F}$. With this definition, a functional $\Phi$ is an \SCF if and only if $\mathscr{F}_{\Phi(\tau_{0:n})}$ is an \SCSA. The corresponding importance sampling estimator is then given by
\begin{align*}
    \mathbb{E}_{\trajectory{\mu}|_{(x, a)}}\!\left\lbrack \frac{\partialtrajectory{\pi}{0:n}|_{(x, a)}(\tau_{0:n})}{\partialtrajectory{\mu}{0:n}|_{(x, a)}(\tau_{0:n})} \middle|\mathscr{F} \right\rbrack \Psi(\tau_{0:n}) \, .
\end{align*}
The analogue of the preorder $\precsim$ over conditioning functionals is the \emph{inclusion partial order} $\subseteq$ over sigma-algebras; we have $\Phi_1 \precsim \Phi_2$ if and only if $\mathscr{F}_{\Phi_1(\tau_{0:n})} \subseteq \mathscr{F}_{\Phi_2(\tau_{0:n})}$. 
Further, if for two conditioning functionals $\Phi_1$ and $\Phi_2$ we have $\Phi_1 \precsim \Phi_2$ and $\Phi_2 \precsim \Phi_1$ (that is, roughly speaking, $\Phi_1$ and $\Phi_2$ encode the same information about the trajectory), then we have $\mathscr{F}_{\Phi_1(\tau_{0:n})} = \mathscr{F}_{\Phi_2(\tau_{0:n})}$. Thus, working with sigma-algebras eliminates the issue of several conditioning objects representing exactly the same information about the trajectory.

\subsection{Generalising the conditional importance sampling framework to other classes of MDPs}\label{sec:generalising}

We have restricted the presentation in the main paper to MDPs with finite state and action spaces and reward distributions with finite support
for ease of exposition, and to avoid having to introduce measure-theoretic terminology such as Radon-Nikodym derivatives to deal with more general classes of MDPs. Nevertheless, the framework described in the main paper applies much more generally, such as for certain classes of MDPs with continuous state and/or action spaces. In this section, we briefly describe how the framework generalises to these settings. The aim is not to be exhaustive, but rather to indicate how key concepts change when moving away from the assumptions of the main paper; for a rigorous treatment of the measure-theoretic issues that arise in MDPs with more general state and action spaces, see \citet{bertsekas2007stochastic}.

Consider now an MDP with a general state space $\statespace$ and action space $\actionspace$, each equipped with a sigma-algebra, and consider $\mathbb{R}$, the domain of rewards in the MDP, to be equipped with its usual Borel sigma-algebra. Given measurable transition kernel $P : \statespace\times\actionspace \rightarrow \mathscr{P}(\statespace)$, reward kernel $\mathcal{R} : \statespace\times\actionspace \rightarrow \mathscr{P}(\mathbb{R})$, initial state distribution $\nu \in \mathscr{P}(\statespace)$, and two Markov policies $\pi, \mu : \statespace \rightarrow \mathscr{P}(\actionspace)$, we can straightforwardly define trajectory measures $\partialtrajectory{\mu}{0:n}$, $\partialtrajectory{\pi}{0:n}$, and conditional trajectory measures $\partialtrajectory{\mu}{0:n}|_{(x, a)}$, $\partialtrajectory{\pi}{0:n}|_{(x, a)}$ over the relevant product space. 
The key requirement in order to be able to carry out importance sampling in this more general case is that $\partialtrajectory{\pi}{0:n}|_{(x, a)}$ is absolutely continuous with respect to $\partialtrajectory{\mu}{0:n}|_{(x, a)}$. When this is the case, the Radon-Nikodym derivative
\begin{align*}
    \frac{\mathrm{d}\partialtrajectory{\pi}{0:n}|_{(x, a)}}{\mathrm{d}\partialtrajectory{\mu}{0:n}|_{(x, a)}}(\tau_{0:n})
\end{align*}
exists, and has the property that for a measurable functional $\Psi$ of the trajectory, under mild integrability conditions, we have
\begin{align*}
    \mathbb{E}_{\partialtrajectory{\mu}{0:n}|_{(x, a)}}\!\left\lbrack \frac{\mathrm{d}\partialtrajectory{\pi}{0:n}|_{(x, a)}}{\mathrm{d}\partialtrajectory{\mu}{0:n}|_{(x, a)}}(\tau_{0:n}) \Psi(\tau_{0:n}) \right\rbrack
     = 
    \mathbb{E}_{\partialtrajectory{\pi}{0:n}|_{(x, a)}}\!\left\lbrack \Psi(\tau_{0:n}) \right\rbrack \, ,
\end{align*}
the fundamental property we require an importance weight to satisfy. The \CIW framework of the main paper can thus be extended to these more general settings by computing conditional expectations of the Radon-Nikodym derivative of the two trajectory measures. We conclude by noting that in several practical applications of interest, $\statespace$ and $\actionspace$ are themselves subsets of Euclidean spaces, with $\pi(\cdot|x)$ and $\mu(\cdot|x)$ taken to have densities with respect to Lebesgue measure for each $x \in \statespace$; in such circumstances, under mild assumptions, the Radon-Nikodym derivative can be expressed in the familiar form of a product of action density ratios; that is
\begin{align*}
    \frac{\mathrm{d}\partialtrajectory{\pi}{0:n}|_{(x, a)}}{\mathrm{d}\partialtrajectory{\mu}{0:n}|_{(x, a)}}(\tau_{0:n}) = \prod_{t=1}^{n-1} \frac{\pi(A_t|X_t)}{\mu(A_t|X_t)} \, .
\end{align*}
However, in cases where the action distribution $\pi(\cdot|x)$ is \emph{not} absolutely continuous with respect to $\mu(\cdot|x)$, such as in deterministic policy gradient algorithms \citep{silver2014deterministic,lillicrap2015continuous}, the measure $\partialtrajectory{\pi}{0:n}$ is \emph{not} absolutely continuous with respect to $\partialtrajectory{\mu}{0:n}$, meaning that the Radon-Nikodym derivative does not exist, and so importance sampling, and in particular the \CIW framework, cannot straightforwardly be applied.

\end{document}